%% file: main.tex
\documentclass{template}

\usepackage{cite}

\ifCLASSINFOpdf
   \usepackage[pdftex]{graphicx}
\else

   \usepackage[dvips]{graphicx}
\fi

\usepackage{amsfonts,amsmath,amssymb,amsthm}
\usepackage[super]{nth}

\newtheorem{theorem}{Theorem}

\usepackage{algorithmic}
\usepackage[ruled,linesnumbered,noresetcount,vlined]{algorithm2e}

\ifCLASSOPTIONcompsoc
 \usepackage[caption=false,font=normalsize,labelfont=sf,textfont=sf]{subfig}
\else
 \usepackage[caption=false,font=footnotesize]{subfig}
\fi

\usepackage{float}
\floatstyle{ruled}
\newfloat{model}{thp}{lop}
\floatname{model}{Model}

\usepackage{enumitem}
\newlist{abbrv}{itemize}{1}
\setlist[abbrv,1]{label=,labelwidth=0.5in,align=parleft,itemsep=0.1\baselineskip,leftmargin=!}

\usepackage{booktabs}
\usepackage{csvsimple}

\usepackage{siunitx}

\input{tex/definitions}

\begin{document}

\title{Learning Optimal Power Flow Value Functions\\ with Input-Convex Neural Networks}

\author{\IEEEauthorblockN{Andrew Rosemberg\IEEEauthorrefmark{1},
Mathieu Tanneau\IEEEauthorrefmark{1},
Bruno Fanzeres\IEEEauthorrefmark{2},
Joaquim Garcia\IEEEauthorrefmark{3} and
Pascal Van Hentenryck\IEEEauthorrefmark{1}}
\IEEEauthorblockA{\IEEEauthorrefmark{1}\textit{Georgia Institute of Technology}, 
arosemberg3@gatech.edu, \{mathieu.tanneau, pascal.vanhentenryck\}@isye.gatech.edu}
\IEEEauthorblockA{\IEEEauthorrefmark{2} \textit{Industrial Engineering Department}, \textit{Pontifical Catholic University of Rio de Janeiro}, bruno.santos@puc-rio.br}
\IEEEauthorblockA{\IEEEauthorrefmark{3} Research and Development, PSR - Energy Consulting and Analytics, joaquimgarcia@psr-inc.com}
}

\maketitle

\begin{abstract}
The Optimal Power Flow (OPF) problem is integral to the functioning of power systems, aiming to optimize generation dispatch while adhering to technical and operational constraints. These constraints are far from straightforward; they involve intricate, non-convex considerations related to Alternating Current (AC) power flow, which are essential for the safety and practicality of electrical grids. However, solving the OPF problem for varying conditions within stringent time-frames poses practical challenges. 
To address this, operators often resort to model simplifications of varying accuracy. Unfortunately, better approximations (tight convex relaxations) are often still computationally intractable. This research explores machine learning (ML) to learn convex approximate solutions for faster analysis in the online setting while still allowing for coupling into other convex dependent decision problems. By trading off a small amount of accuracy for substantial gains in speed, they enable the efficient exploration of vast solution spaces in these complex problems.
\end{abstract}


\thanksto{\noindent This research is partly funded by NSF award 2112533.}

\input{tex/nomenclature}

\input{tex/introduction}

\input{tex/literature}

\input{tex/opf}

\input{tex/methodology}

\input{tex/results}

\section{Conclusion}
\label{section: conclusion}

This paper studies whether ICNNs can be accurate enough to approximate
the value function of large-scale OPF problems. ICNNs are an
attractive avenue for approximating OPF problems, since their
convexity enables them to be embedded in optimization models without
resorting to a MIP reformulation. The results in this paper indicate
that ICNNs can consistently approximate system operating costs under
different grid conditions, often with optimization gaps below $0.5$.
This suggests that convexity constraints do not entail a substantial
sacrifice in accuracy, even in inherently non-convex scenarios like
AC-OPF.  Notably, in most cases, ICNNs exhibit slightly lower
geometric errors than DNNs with similar architectures, reinforcing
the potential of this approach.

It is important to acknowledge challenges and opportunities in ICNN
research. As highlighted by the SOC-OPF formulation of the
$\texttt{2869\_pegase}$ system, ICNN training did not converge
smoothly in some instances, indicating that fine tuning the
optimization parameters may be needed to render more efficient
training. Additionally, convexity introduced a slight optimistic bias
increasing the error in extreme cases.  Although it may not result in
large impacts, and might even be desirable depending on the
application, it is important to mitigate its potential negative
effects.  Penalizing differently for under and over estimations might
allow the training to be tailored for different applications.
    
In summary, this paper identified a promising pathway for addressing
complex and time-sensitive OPF couplings in power systems.  The
ability to achieve near-optimal cost estimations with reduced
computational burden through machine learning methods has significant
implications for the efficient operation and management of electrical
grids. As research on ICNNs for OPF progresses, it is likely that they
will play an increasingly vital role in ensuring the reliability and
sustainability of power systems in the future.

\bibliographystyle{IEEEtran}
\bibliography{blib}

\end{document}

%% file: tex/definitions.tex
\makeatletter
\DeclareRobustCommand{\cev}[1]{%
  {\mathpalette\do@cev{#1}}%
}
\newcommand{\do@cev}[2]{%
  \vbox{\offinterlineskip
    \sbox\z@{$\m@th#1 x$}%
    \ialign{##\cr
      \hidewidth\reflectbox{$\m@th#1\vec{}\mkern4mu$}\hidewidth\cr
      \noalign{\kern-\ht\z@}
      $\m@th#1#2$\cr
    }%
  }%
}
\makeatother

\DeclareMathOperator{\relu}{ReLU}
\DeclareMathOperator{\diam}{diam}

\newcommand{\im}{\mathbf{j}}

\newcommand{\pg}{\mathbf{p}^{\text{g}}}
\newcommand{\qg}{\mathbf{q}^{\text{g}}}
\newcommand{\Sg}{\mathbf{S}^{\text{g}}}

\newcommand{\vm}{\mathbf{v}}
\newcommand{\va}{\mathbf{\theta}}
\newcommand{\V}{\mathbf{V}}

\newcommand{\pf}{{\mathbf{p}}^{\text{f}}}
\newcommand{\qf}{{\mathbf{q}}^{\text{f}}}
\newcommand{\Sf}{{\mathbf{S}}^{\text{f}}}

\newcommand{\wmsoc}{\mathbf{w}}
\newcommand{\wrsoc}{\mathbf{w}^{\text{r}}}
\newcommand{\wisoc}{\mathbf{w}^{\text{i}}}

\newcommand{\pd}{\text{p}^{\text{d}}}
\newcommand{\qd}{\text{q}^{\text{d}}}
\newcommand{\Sd}{\text{S}^{\text{d}}}
\newcommand{\pgmin}{\underline{\text{p}}^{\text{g}}}
\newcommand{\pgmax}{\overline{\text{p}}^{\text{g}}}
\newcommand{\qgmin}{\underline{\text{q}}^{\text{g}}}
\newcommand{\qgmax}{\overline{\text{q}}^{\text{g}}}
\newcommand{\vmmin}{\underline{\text{v}}}
\newcommand{\vmmax}{\overline{\text{v}}}



%% file: tex/nomenclature.tex
\section*{Nomenclature}
\label{sec:nomenclature}

\begin{abbrv}
    \item[$\mathcal{N}$] Set of buses
    \item[$\mathcal{E}$] Set of branches
    \item[$\mathcal{E}^{R}$] Set of reverse branches
\end{abbrv}


\begin{abbrv}
    \item[$\im$] Imaginary unit $\im^{2} = -1$
    \item[$z^{\star}$] Complex conjugate of $z$
    \item[$\Sd$] Complex power demand; $\Sd = \pd + \im \qd$
    \item[$Y^{s}$] Bus shunt admittance
    \item[$Y$] Complex branch line admittance
    \item[$Y^{c}$] Complex branch shunt admittance
    \item[$\bar{s}$] Thermal branch limit
    \item[$\pgmin, \pgmax$] Active power generation bounds
    \item[$\qgmin, \qgmax$] Reactive power generation bounds 
    \item[$\vmmin, \vmmax$] Voltage magnitude bounds
    \item[$b$] Branch susceptance
\end{abbrv}


\begin{abbrv}
    \item[$\Sg$] Complex power generation; $\Sg {=} \pg + \im \qg$
    \item[$\Sf$] Complex power flow; $\Sf = \pf + \im \qf$
    \item[$\V$] Complex voltage; $\V = \vm \angle \theta$
\end{abbrv}

%% file: tex/introduction.tex
\section{Introduction}

The Optimal Power Flow (OPF) problem optimizes generation dispatch
while satisfying physical and engineering constraints.  It is
therefore fundamental for many aspects of power systems operations:
market-clearing, unit commitment, optimal transmission switching,
transmission expansion planning, to name a few. Its alternating
current (AC) form \cite{carpentier1979optimal, granville1996application}, AC-OPF, is a nonlinear, non-convex problem which
makes it challenging to solve in practice, especially when combined
with discrete decisions like unit commitment, line switching, and bus
splitting.  Therefore, operators rely on approximate OPF formulations,
typically the DC-OPF approximation which, although more tractable, may
lead to sub-optimal or unsafe decisions when far from the traditional
operating point, because it does not capture the complexity of AC
systems.

These computational limitations have spurred interest in optimization
proxies for power systems, and OPF problems in particular.
Optimization proxies \cite{kotary2021end, van2021machine, velloso2021combining} are machine learning (ML) models that approximate
the input-to-output mapping of an optimization solver; once trained,
they produce predictions in milliseconds.  A large body of work has
focused on predicting solutions to OPF problems, especially DC-OPF and
AC-OPF. In this case, the proxy takes the input data of the OPF as
input and outputs a near-feasible, near-optimal solution.  This
enables real-time risk assessment at massive scales.

Another stream of research uses optimization proxies to capture
complex interactions, e.g., AC power flow equations, then embeds the
trained proxy in a larger optimization problem, e.g., a
unit-commitment problem
\cite{Kody2022_DNNUnitCommitment,Wu2023_TransientUCusingICNN}.  This
strategy replaces the nonlinear component. e.g., the AC power flow
equations, with a mixed-integer representation of a trained neural
network. Although it removes the nonconvexity stemming from the
physics, this approach requires the use of discrete variables,
introducing another type of non-convexity, which reduces its
tractability for large-scale systems.

To address this challenge, this paper explores the use of {\em
  input-convex neural networks (ICNN)} \cite{amos2017input} as an alternative to non-convex
DNN for applications where the neural network must be embedded in a
larger optimization.  Specifically, the paper focuses on {\em learning
  a tractable approximation of the value function of an OPF problem
  using ICNNs}. Such ICNNs, if accurate enough, would provide a highly
valuable tool for a broad range of applications in power systems: unit
commitment, transmission switching but also a wealth of stochastic
optimization and reinforcement learning methods that implicitly rely
on value functions and their gradients.  

{\em The main objective of the paper is thus to determine whether
  ICNNs, despite their more limited expressive power, can match the
  performance of DNNs for approximating the value of OPF problems}.
This is a pre-requisite to using ICNNs in larger optimizations and an
open issue in the representation power of neural networks. The main
contributions of the paper can be summarized as follows.
\begin{enumerate}
\item The paper contributes strong theoretical guarantees on the
  performance of ICNNs. In particular, it provides bounds on the
  generalization error of ICNNs that only depend on the ICNN
  performance on the training data. 

\item The paper explores specific ICNN architectures and trains it
  to learn three OPF formulations: the AC-OPF, the SOC relaxation, and
  the DC-OPF.

\item The paper reports the performance of the resulting ICNNs on
  large-scale systems, that are 50 times larger than prior
  research. The results demonstrate that ICCNs are capable of learning
  the value function of OPF problems, at least as effectively as DNNs
  on traditional test cases, with optimality gaps almost always lower
  than 0.5\%.
\end{enumerate}

\noindent
The rest of the paper is organized as follows. Section \ref{section:
  Literature} reviews related works in the literature, Section
\ref{sec:OPF} describes the OPF formulations considered in the paper.
Section \ref{sec:ICNN} presents the input-convex architecture used in
the paper, and provides strong generalization bounds for this class of
models.  Section \ref{sec:results} presents and analyses the results
of the proposed comparison, and Section \ref{section: conclusion}
concludes the paper.

%% file: tex/literature.tex
\section{Related Work}
\label{section: Literature}

The decomposition of intricate problems through value function approximations has found widespread application both in industry and academic literature. This approach has been instrumental in achieving tractable solutions for various practical scenarios, including multistage decision-making problems.

In multistage problems, such as those encountered in storage management and long-term asset investment, decision-makers often seek optimal policies. They do so by employing a spectrum of function approximations that range from simple parametric forms \cite{powell2022parametric} to more intricate piece-wise models \cite{pereira1991multi, shapiro2009lectures}. Some of these advanced models may involve a substantial number of individual function evaluations to reach convergence \cite{ruszczynski2011nonlinear, nesterov2003introductory, asamov2018regularized}.

DNNs have emerged as a standout player in approximating solutions, particularly in the domain of Optimal Power Flow (OPF) problems \cite{chen2022learning, fioretto2020predicting, mak2021load}. Furthermore, He et al. \cite{he2022enabling} have successfully employed neural networks to approximate the cost function of unit commitment problems, streamlining constraint screening processes and improving solution efficiency.

ICNNs have also found applications in energy-related challenges, such as unit commitment \cite{Wu2023_TransientUCusingICNN} and voltage regulation \cite{chen2020data}. These networks offer a unique advantage by ensuring convexity within specific regions of their input domains through parameter constraints, thereby reducing the complexity of identifying convex mappings.

A notable contribution by Zhang et al. \cite{Zhang2022_ConvexNNSolverDCOPF} involves training convex neural networks to predict the objective values of DC-OPF. This innovative approach enables the derivation of dual solutions, aiding in the identification of active sets of constraints. Leveraging the convexity of these ICNNs, this method augments the training process and provides valuable generalization bounds. Similar investigations by Chen et al. \cite{Chen2022_LearningDCOPFDuality} extend this approach to systems with up to 118 buses, highlighting its applicability.

In another pioneering effort, Wu et al. \cite{Wu2023_TransientUCusingICNN} harness the power of ICNNs to map pre-fault operation conditions to transient stability indices. This enables the formulation of transient stability constraints, with numerical experiments conducted on systems featuring 39 and 118 buses validating the effectiveness of their methodology.

Machine Learning (ML) techniques, including ICNN-based approaches, have shown promise in discovering convex approximations and relaxations for optimization problems. In the context of Optimal Power Flow (OPF) applications \cite{Cengil2022_AccelerateGlobalOptimalSolutions}, methodologies utilizing ICNNs have exhibited significant potential \cite{duchesne2021supervised}. These collective advances underscore the growing role of advanced neural network techniques in enhancing optimization methodologies, promising more efficient and effective problem-solving strategies.

{\em This paper extends these lines of research in several
directions. First, it demonstrates, for the first time, that ICNNs can
provide state-of-the-art results in predicting value functions for
large-scale OPF problems involving thousands of buses. The OPF
problems studied in the paper also go beyond the DC model and include
the SCOP relaxation and the AC-OPF. Second, the paper contributes
strong generalization bounds for ICNNs that significantly expand
existing work.}

%% file: tex/opf.tex
\section{Optimal Power Flow}
\label{sec:OPF}

The Optimal Power Flow (OPF) problem \cite{carpentier1962contribution} is a fundamental problem in power systems operations.
The OPF problem finds the most economical generation dispatch so as to serve electricity demand while satisfying physical and engineering constraints.
The paper considers the AC-OPF formulation, its second-order cone (SOC) relaxation, and its DC-OPF linear approximation, which are presented next.
For ease of reading, the presentation omits transformer tap ratios, phase angle difference constraints, and reference voltage (slack bus) constraints.
All are implemented and considered in the experiments of Section \ref{sec:results}

\subsection{The AC-OPF Formulation}
\label{sec:OPF:AC}

    \begin{model}[!t]
        \caption{The AC-OPF Model}
        \label{model:ACOPF}
        \begin{subequations}
        \label{eq:ACOPF}
        \footnotesize
        \begin{align}
            \min \quad 
            & \sum_{i \in \mathcal{N}} c_{i} \pg_{i} \label{eq:ACOPF:objective}\\
            \textrm{s.t.} \quad
                & \Sg_{i} - \Sd_{i} - (Y^{s}_{i})^{\star} |\V_{i}|^{2} = \sum_{ij \in \mathcal{E} \cup \mathcal{E}^{R}} \Sf_{ij}
                && \forall i \in \mathcal{N} 
                \label{eq:ACOPF:kirchhoff} \\
            & \Sf_{ij} = (Y_{ij} + Y_{ij}^{c})^{\star} |\V_{i}|^{2} - Y_{ij}^{\star} \V_{i} \V_{j}^{\star}
                && \forall ij \in \mathcal{E} \label{eq:ACOPF:ohm_fr}\\
            & \Sf_{ji} = (Y_{ij} + Y_{ji}^{c})^{\star} |\V_{j}|^{2} - Y_{ij}^{\star} \V_{i}^{\star} \V_{j}
                && \forall ij \in \mathcal{E} \label{eq:ACOPF:ohm_to}\\
            & |\Sf_{ij}|, |\Sf_{ji}| \leq \bar{s}_{ij}
                && \forall ij \in \mathcal{E} \label{eq:ACOPF:thermal_limits} \\
            & \vmmin_{i} \leq |\V_{i}| \leq \vmmax_{i} 
                && \forall i \in \mathcal{N} \label{eq:ACOPF:voltage_bounds}\\
            & \pgmin_{i} \leq \pg_{i} \leq \pgmax_{i}
                && \forall i \in \mathcal{N} \label{eq:ACOPF:active_dispatch_bounds}\\
            & \qgmin_{i} \leq \qg_{i} \leq \qgmax_{i} 
                && \forall i \in \mathcal{N} \label{eq:ACOPF:reactive_dispatch_bounds}
        \end{align}
        \end{subequations}
    \end{model}

    Model \ref{model:ACOPF} presents the AC-OPF formulation, in complex variables.
    The objective \eqref{eq:ACOPF:objective} minimizes total generation costs.
    Constraints \eqref{eq:ACOPF:kirchhoff} enforce power balance (Kirchhoff's current law) at each bus.
    Constraints \eqref{eq:ACOPF:ohm_fr} and \eqref{eq:ACOPF:ohm_to} express Ohm's law on forward and reverse power flows, respectively.
    Constraints \eqref{eq:ACOPF:thermal_limits} enforce thermal limits on forward and reverse power flows.
    Finally, constraints \eqref{eq:ACOPF:voltage_bounds}--\eqref{eq:ACOPF:reactive_dispatch_bounds} enforce minimum and maximum limits on nodal voltage magnitude, active generation, and reactive generation, respectively.
    The AC-OPF problem is nonlinear and non-convex and is typically solved using interior-point algorithms.

\subsection{The SOC-OPF Formulation}
\label{sec:OPF:SOC}

    \begin{model}[!t]
        \caption{The SOC-OPF model}
        \label{model:SOCOPF}
        \begin{subequations}
        \label{eq:SOCPF}
        \footnotesize
        \begin{align}
            \min \quad 
            & \sum_{i \in \mathcal{N}} c_{i} \pg_{i} \label{eq:SOCOPF:objective}\\
            \textrm{s.t.} \quad
            & \pg_{i} - \pd_{i} - g^{s}_{i} \wmsoc_{i} = \sum_{ij \in \mathcal{E} \cup \mathcal{E}^{R}} \pf_{ij}
                && \forall i \in \mathcal{N} 
                \label{eq:SOCOPF:kirchhoff_active} \\
            & \qg_{i} - \qd_{i} + b^{s}_{i} \wmsoc_{i} = \sum_{ij \in \mathcal{E} \cup \mathcal{E}^{R}} \qf_{ij}
                && \forall i \in \mathcal{N} 
                \label{eq:SOCOPF:kirchhoff_reactive} \\
            & \pf_{ij} = \gamma^{p}_{ij} \wmsoc_{i} + \gamma^{p,r}_{ij} \wrsoc_{ij} + \gamma^{p,i}_{ij} \wisoc_{ij}
                && \forall ij \in \mathcal{E} 
                \label{eq:SOCOPF:ohm:active:fr}\\
            & \qf_{ij} = \gamma^{q}_{ij} \wmsoc_{j} + \gamma^{q,r}_{ij} \wrsoc_{ij} + \gamma^{q,i}_{ij} \wisoc_{ij}
                && \forall ij \in \mathcal{E} 
                \label{eq:SOCOPF:ohm:reactive:fr}\\
            & \pf_{ji} = \gamma^{p}_{ji} \wmsoc_{i} + \gamma^{p,r}_{ji} \wrsoc_{ij} + \gamma^{p,i}_{ji} \wisoc_{ij}
                && \forall ij \in \mathcal{E}
                \label{eq:SOCOPF:ohm:active:to}\\
            & \qf_{ji} = \gamma^{q}_{ji} \wmsoc_{j} + \gamma^{q,r}_{ji} \wrsoc_{ij} + \gamma^{q,i}_{ji} \wisoc_{ij}
                && \forall ij \in \mathcal{E}
                \label{eq:SOCOPF:ohm:reactive:to}\\
            & (\pf_{ij})^{2} + (\qf_{ij})^{2} \leq \bar{s}_{ij}^{2}
                && \forall ij \in \mathcal{E} \cup \mathcal{E}^{R} 
                \label{eq:SOCOPF:thermal_limits} \\
            & (\wrsoc_{ij})^{2} + (\wisoc_{ij})^{2} \leq \wmsoc_{i} \wmsoc_{j}
                && \forall ij \in\mathcal{E}
                \label{eq:SOCOPF:jabr}\\
            & \vmmin_{i}^{2} \leq \wmsoc_{i} \leq \vmmax_{i}^{2} 
                && \forall i \in \mathcal{N} \label{eq:SOCOPF:voltage_bounds}\\
            & \pgmin_{i} \leq \pg_{i} \leq \pgmax_{i}
                && \forall i \in \mathcal{N} \label{eq:SOCOPF:active_dispatch_bounds}\\
            & \qgmin_{i} \leq \qg_{i} \leq \qgmax_{i} 
                && \forall i \in \mathcal{N} \label{eq:SOCOPF:reactive_dispatch_bounds}
        \end{align}
        \end{subequations}
    \end{model}

    The Second-Oder Cone (SOC) relaxation of AC-OPF proposed by Jabr \cite{Jabr2006_SOC-OPF} is obtained from AC-OPF by introducing the additional variables
    \begin{align}
        \wmsoc_{i} &= \vm_{i}^{2}, 
            && \forall i \in \mathcal{N}\\
        \wrsoc_{ij} &= \vm_{i} \vm_{j} \cos (\theta_{j} - \theta_{i}), 
            && \forall ij \in \mathcal{E}\\
        \wisoc_{ij} &= \vm_{i} \vm_{j} \sin(\theta_{j} - \theta_{i}), 
            && \forall ij \in \mathcal{E}
    \end{align}
    and the non-convex constraint
    \begin{align}
        \label{eq:SOC:quad_prod}
        (\wrsoc_{ij})^{2} + (\wisoc_{ij})^{2} &= \wmsoc_{i} \wmsoc_{j}, && \forall ij \in \mathcal{E}.
    \end{align}
    The SOC relaxation is then obtained by relaxing Eq. \eqref{eq:SOC:quad_prod} into
    \begin{align}
        \label{eq:SOC:Jabr}
        (\wrsoc_{ij})^{2} + (\wisoc_{ij})^{2} &\leq \wmsoc_{i} \wmsoc_{j}, && \forall ij \in \mathcal{E}.
    \end{align}

    Model \ref{model:SOCOPF} presents the SOC-OPF formulation, in real variables.
    The objective function \eqref{eq:SOCOPF:objective} is equivalent to \eqref{eq:ACOPF:objective}.
    Constraints \eqref{eq:SOCOPF:kirchhoff_active} and \eqref{eq:SOCOPF:kirchhoff_reactive} enforce, at each node, Kirchhoff's current law for active and reactive power, respectively.
    Constraints \eqref{eq:SOCOPF:ohm:active:fr}--\eqref{eq:SOCOPF:ohm:active:to} capture Ohm's law on active and reactive, forward and reverse power flows.
    The $\gamma$ parameters that appear in the constraints are constant terms derived from \eqref{eq:ACOPF:ohm_fr}--\eqref{eq:ACOPF:ohm_to}, after substituting variables $\wmsoc, \wrsoc, \wisoc$.
    Constraints \eqref{eq:SOCOPF:thermal_limits} enforce thermal limits on forward and reverse power flows.
    Constraints \eqref{eq:SOCOPF:jabr} is Jabr's inequality \eqref{eq:SOC:Jabr}.
    Finally, constraints \eqref{eq:SOCOPF:voltage_bounds}--\eqref{eq:SOCOPF:reactive_dispatch_bounds}, like constraints \eqref{eq:ACOPF:voltage_bounds}--\eqref{eq:ACOPF:reactive_dispatch_bounds}, enforce minimum and maximum limits on nodal voltage magnitude, active and reactive generation.

    The SOC-OPF formulation is nonlinear and convex.
    This makes it more tractable to solve than AC-OPF using, e.g., polynomial-time interior-point algorithms.
    Furthermore, since it is a relaxation of AC-OPF, solving SOC-OPF provides valid dual bounds on the optimal value of AC-OPF.

\subsection{The DC-OPF Formulation}
\label{sec:OPF:DC}

    The DC-OPF formulation is a linear approximation of AC-OPF.
    The approximation assumes that all voltage magnitudes are one per-unit, voltage angles are small and losses are negligible, and it ignores reactive power \cite{Stott2009_DC,taylor2015convex}.
    The DC approximation underlies virtually all electricity markets and is widely used in, e.g., unit commitment and transmission network expansion planning problems.

    \begin{model}[!t]
        \caption{The DC-OPF Model}
        \label{model:DCOPF}
        \begin{subequations}
        \label{eq:DCOPF}
        \footnotesize
        \begin{align}
            \min \quad & \sum_{i \in \mathcal{N}} c_{i} \pg_{i} \label{eq:DCOPF:obj}\\
            \text{s.t.} \quad
                & \pg_{i} + \sum_{ji \in \mathcal{E}} \pf_{ji} - \sum_{ij \in \mathcal{E}} \pf_{ij}  = \pd_{i}
                    & \forall i \in \mathcal{N}
                   \label{eq:DCOPF:power_balance}\\
                & \pf_{ij} = b_{ij} (\va_{j} - \va_{i}) 
                   & \forall ij \in \mathcal{E}
                   \label{eq:DCOPF:ohm}\\
                & |\pf_{ij}| \leq \bar{s}_{ij}
                    & \forall ij \in \mathcal{E}
                    \label{eq:DCOPF:bounds:pf}\\
                & \pgmin_{i} \leq \pg_{i} \leq \pgmax_{i} 
                    & \forall i \in \mathcal{N}
                   \label{eq:DCOPF:bounds:pg}
        \end{align}
        \end{subequations}
    \end{model}
    
    Model \ref{model:DCOPF} presents the resulting linear programming (LP) formulation of DC-OPF.
    The objective \eqref{eq:DCOPF:obj} minimizes total generation costs.
    Constraints \eqref{eq:DCOPF:power_balance} enforce (active) power balance at each node.
    Constraints \eqref{eq:DCOPF:ohm} approximate Ohm's law using a phase-angle formulation.
    Note that, DC-OPF does not consider reverse power flows, unlike AC-OPF and SOC-OPF; this is because losses are neglected in DC-OPF.
    Constraints \eqref{eq:DCOPF:bounds:pf} enforce thermal constraints on each branch.
    Constraints \eqref{eq:DCOPF:bounds:pg} enforce minimum and maximum limits on active power generation.
    Finally, recall that constraints on phase angle differences and slack bus are omitted from the presentation for readability, but are implemented in all the numerical experiments.

%% file: tex/methodology.tex
\section{Methodology}
\label{sec:ICNN}

The goal of the paper is to train an ML model that takes the nodal
demand vector $\Sd$ as input, and outputs the optimal solution of its
corresponding OPF problem.  This section presents the input-convex
neural network (ICNN) architecture and its training, and establishes
its generalization guarantees.

\subsection{The Input-Convex Neural Network Architecture}
\label{sec:ICNN:architecture}

An ICNN is a special type of DNN that computes a convex function of its
input \cite{amos2017input}.  ICNNs are well-suited when one seeks to
represent or approximate convex functions since they are convex by
design, unlike general DNNs.  ICNNs achieve their convexity by
combining convex activation functions with convexity-preserving
operations \cite{amos2017input}.

The simplest ICNN architecture consists of fully-connected layers of
the form $h(\mathbf{x}) \, {=} \, \relu \left( W \mathbf{x} + d \right)$,
where $\mathbf{x} \, {\in} \, \mathbb{R}^{n}$ denotes the layer input
vector, $d \, {\in} \, \mathbb{R}^{m}$ is the bias vector, and $W \, {\in}
\, \mathbb{R}^{m \times n}$ is a weight matrix with non-negative coefficients, to ensure convexity.
The Rectified Linear Unit ($\relu$) activation function 
$\relu(x) \,{=}\, \max(0, x)$ is applied element-wise.
Figure \ref{fig:ICNN:example} illustrates the difference between DNNs and ICNNs on a small example, and showcases ICNNs' convexity.

\begin{figure}[!t]
    \centering
    \subfloat[Example DNN (left) and its Output (right). The DNN defines a non-convex function.]{
        \centering
        \includegraphics[width=0.45\columnwidth]{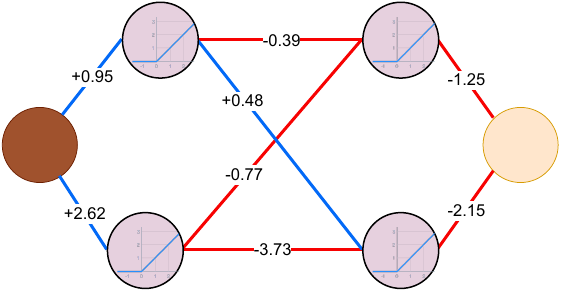}
        \hspace{1em} 
        \includegraphics[width=0.45\columnwidth]{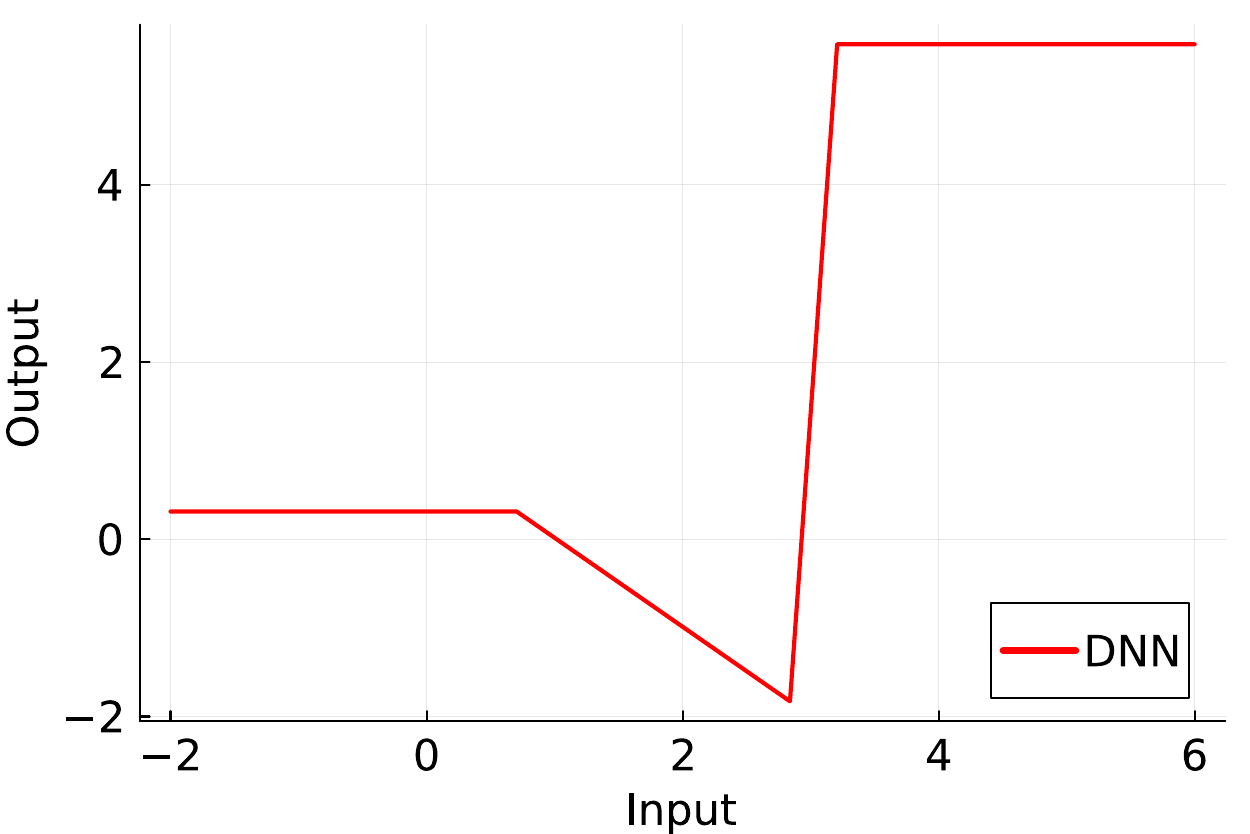}
    }\\
    \subfloat[Example ICNN (left) and its Output (right). All ICNN weights are positive and it defines a convex function.]{
        \centering
        \includegraphics[width=0.45\columnwidth]{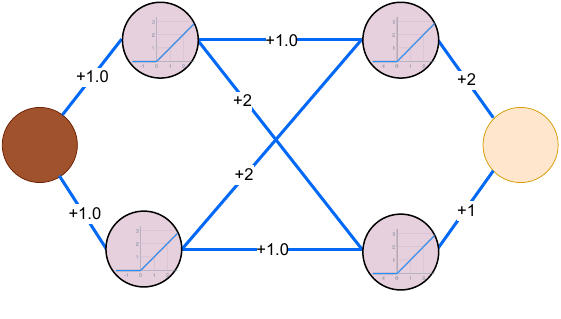}
        \hspace{1em}
        \includegraphics[width=0.45\columnwidth]{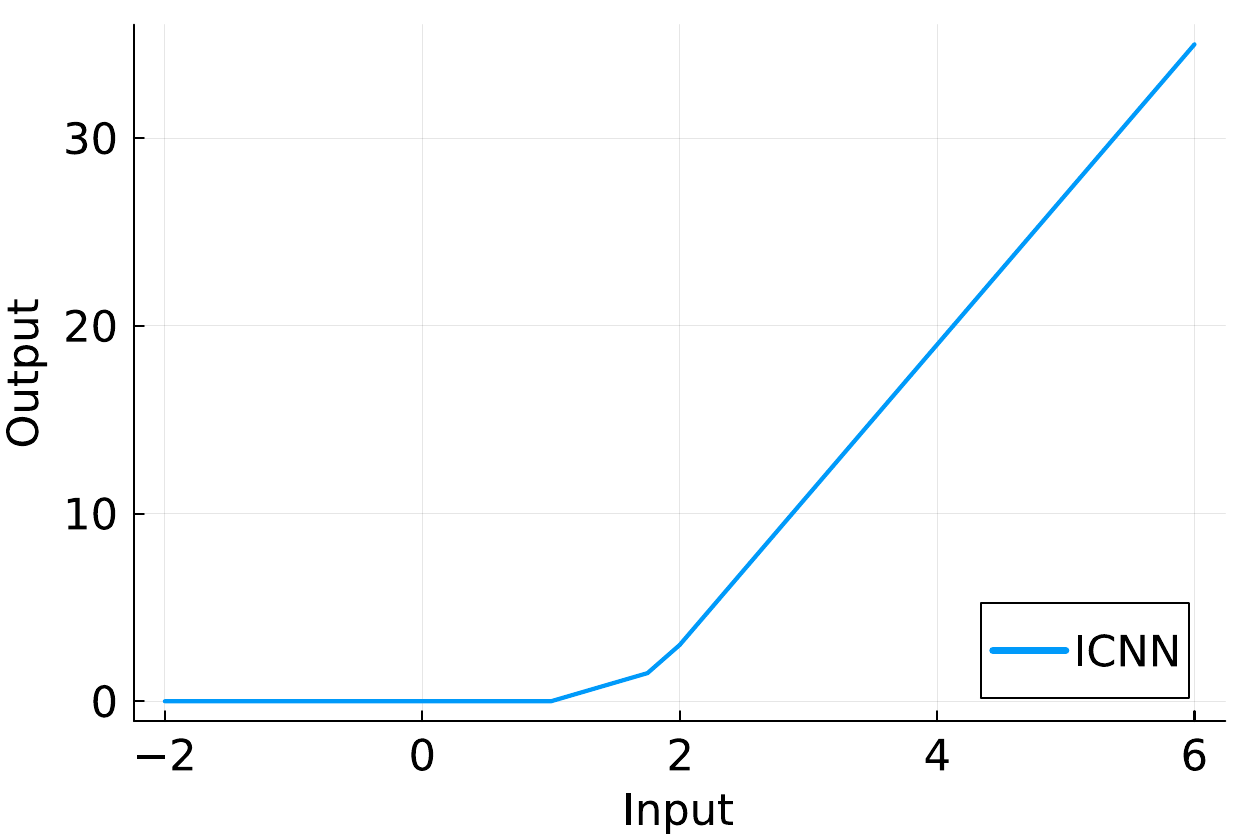}
    }\\
    \caption{Illustration of Input-Convex Neural Networks.}
    \label{fig:ICNN:example}
\end{figure}

To learn the value function of OPF problems, this paper considers ICNN
architectures with skip-connections, which allow the approximation of convex functions with
decreasing slopes, and have been shown to improve performance
\cite{amos2017input,Zhang2022_ConvexNNSolverDCOPF}. The overall
architecture is illustrated in Figure \ref{fig:ICNN:full}. Its
$k$-th layer is of the form
\begin{align}
    \mathbf{x}^{k} = h^{k}(\mathbf{x}^{k-1}) = \relu(W^{k} \mathbf{x}^{k-1} + H^{k} \mathbf{x}^{0} + d^{k}),
\end{align}
where $\mathbf{x}^{k}$ and $\mathbf{x}^{k-1}$ denote the outputs of
layer $k$ and $k{-}1$, $\mathbf{x}^{0}$ denotes the input of the ICNN,
i.e., $\mathbf{x}^{0} \, {=} \, (\pd, \qd)$, $d^{k}$ is the bias vector, and
$W^{k}, H^{k}$ are weight matrices.  Skip-connections feed the ICNN
input $\mathbf{x}^{0}$ to each layer.  The coefficients of $W^{k}$ are
non-negative, whereas $H^{k}$ may take positive or negative values
without affecting convexity \cite{amos2017input}.

\begin{figure}[!t]
    \centering
    \includegraphics[width=0.8\columnwidth]{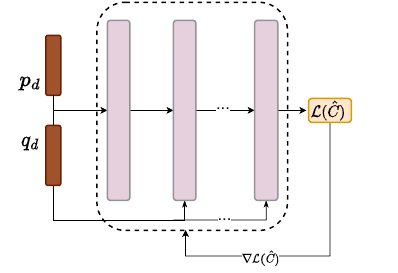}
    \caption{Fully Connected ICNN}
    \label{fig:ICNN:full}
\end{figure}

\subsection{The Training of ICNNs}
\label{sec:ICNN:training}

Because some of their weights must be non-negative, ICNNs cannot
directly be trained with traditional gradient descent algorithms.  A
common choice is to use projected gradient descent, which clips
gradients before updating the weights.  While it converges under the
same assumptions as traditional gradient descent algorithms, projected
gradient descent requires special attention to avoid slow convergence.

Preliminary experiments showed that, for larger learning rates,
gradient clipping results in a dramatic slow-down, when compared to
the training of regular DNNs. This comes from the greater impact on
the loss that happens when projecting back onto the feasible
region. In fact, appropriate learning rates (e.g.,
\cite{tripuraneni2018averaging}) also depend on the relative distance
to the constraint barriers - which can be very small when approaching
the local optimum the algorithm is converging to. Unfortunately,
dynamically adapting learning rates is not an easy task when
post-clipping gradients from a standard gradient descent
algorithms. Scheduling parameters had to be tuned in hyper-parameter
optimization over the validation set.

\subsection{Generalization Guarantees}
\label{sec:ICNN:generalization}

\newcommand{\flo}{\check{f}}
\newcommand{\fup}{\hat{f}}

One fundamental benefit of using ICNNs is their stronger
generalization guarantees.  To the best of the authors' knowledge, the
strongest theoretical guarantee for ICNNs in the context of OPF
proxies is from \cite{Zhang2022_ConvexNNSolverDCOPF} (Theorem 6.5):
they show that the gradient of a perfectly-trained ICNN is bounded
over the convex hull of the training set.  While their proof uses
convexity arguments, the result is not specific to ICNNs. For
instance, it also holds for fully-connected DNNs with ReLU activation
(DNN-ReLU for short). Indeed, DNN-ReLUs represent continuous,
piece-wise linear functions, whose gradient is piece-wise constant.
Since DNN-ReLUs have a finite number of pieces, their gradients are
immediately bounded.

This paper improves on the generalization guarantees of
\cite{Zhang2022_ConvexNNSolverDCOPF}: it provides generalization
bounds for {\em general ICNNs} (Theorem
\ref{thm:ICNN:generalization:general}), and {\em explicit formulae}
for {\em perfectly-trained ICNNs} (Theorem
\ref{thm:ICNN:generalization:perfect}). The results are presented for
a parametric convex optimization problem $\Phi(b)$ that the ICNN $f$
must learn, and its parametric strong dual $\Psi(b)$. For instance,
for the case of a parametric linear program,
\begin{align}
\Phi(b) &= \ \min_{x} \left\{ c^{\top} x \ \middle| \ A x  \geq b \right\},\\
\Psi(b) &= \ \max_{y} \left\{ b^{\top}y \ \middle| \ A^{\top} y  = c, y \geq 0 \right\},
\end{align}
with primal and dual variables $x$ and $y$ respectively.
When strong duality holds, $\Phi(b) = \Psi(b)$.
$\Phi$ is
convex in $b$ and for a given $b$, any dual optimal solution is a
sub-gradient of $\Phi$ \cite{nocedal1999numerical}.  In addition, if
the dual optimum $y^{*}(b)$ is unique, then $\nabla \Phi$ is
defined and $\nabla \Phi(b) = y^{*}(b)$.  

The results are expressed in terms of a dataset $\mathcal{D} \, {=} \,
\{(b_{i}, y_{i}, z_{i})\}_{i=1, ..., N}$ where each $b_{i}, y_{i},
z_{i}$ are the right-hand side, optimal dual solution, and optimal
value of instance $i \,{=}\,1, ..., N$, respectively. They are also
expressed in terms of $\mathcal{B} \,{=}\, \text{conv}\{b_{1}, ...,
b_{N}\}$, the convex hull of the right-hand sides with diameter
$\diam(\mathcal{B})$. The first result is a generalization bound for
arbitrary ICNNs.

\begin{theorem}
\label{thm:ICNN:generalization:general}
Let $f$ be an ICNN, $\tilde{z}_{i} \, {=} \, f(b_{i})$, and $\tilde{y}_{i} \, {=} \, \nabla f(b_{i})$.
There exists a constant $M$, whose value depends only on $\mathcal{D}$ and $\{\tilde{y}_{i}, \tilde{z}_{i}\}$, such that
\begin{align}
\forall b \in \mathcal{B}, |f(b) - \Phi(b)| \leq M.
\end{align}
\end{theorem}
\begin{proof}
        Define $\flo$, $\fup$ over $\mathcal{B}$ as
        \begin{align*}
            \flo(b) &= \max \left\{ \tilde{z}_{i} + \tilde{y}_{i}^{\top}(b - b_{i}) \ \middle| \ i=1, ..., N \right\}\\
            \fup(b) &= \min_{\lambda} \left\{ \sum_{i} \lambda_{i} \tilde{z}_{i} \ \middle| \ \lambda \geq 0, e^{\top} \lambda = 1 \right\} 
        \end{align*}
        Note that $\flo, \fup$ are convex lower and upper envelopes of $f$.
        The constant $M$ is obtained by solving
        \begin{subequations}
        \label{eq:ICNN:generalization_bound:optimization}
        \begin{align}
            M = \max_{b, z} \quad & |z - \Phi(b)|\\ s.t. \quad &
            \flo(b) \leq z \leq \fup(b),\\ & b \in \mathcal{B},
        \end{align}
        \end{subequations}
        which concludes the proof by convexity of $\Phi$.
    \end{proof}

\noindent
Theorem \ref{thm:ICNN:generalization:general} provides a worst-case
guarantee on the generalization performance of the ICNN, which only
depends on the value of $f$ on the dataset $\mathcal{D}$. The next
theorem \ref{thm:ICNN:generalization:perfect} provides an explicit
worst-case guarantee when the ICNN {\em perfectly fits} the training
data.

\begin{theorem}
\label{thm:ICNN:generalization:perfect}
Let $f$ be an ICNN for learning $\Phi$ and assume that, for all $i \in \{1, ..., N\}$,
\begin{align*}
f(b_{i}) = z_{i} = \Phi(b_i) \text{ and } \nabla_{b} f(b_{i}) = y_{i} = \Phi(b_i).
\end{align*}
Then, $\forall b \in \mathcal{B}, |f(b) - \Phi(b) | \leq \max_{i} \|y_{i}\| \times \diam(\mathcal{B})$.
\end{theorem}
\begin{proof}
        Let $b \in \mathcal{B}$, i.e., $b = \sum_{i} \lambda_{i} b_{i}$ where weights $\lambda_{i}$ are non-negative and sum to $1$.
        By convexity of $f$ and $\Phi$, 
        \begin{align*}
            \forall i, \quad z_{i} + (b - b_{i})^{\top} y_{i} \leq \Phi(b), f(b) \leq \sum_{j} \lambda_{j} z_{j}.
        \end{align*}
        Combining the above inequalities with weights $\lambda_{i}$ yields
        \begin{align*}
            |\Phi(b) - f(b) |
                & \leq \sum_{j} \lambda_{j} z_{j} - \sum_{i} \lambda_{i} [z_{i} - (b - b_{i})^{\top}y_{i} ]\\
                & \leq \sum_{i} \lambda_{i} (b - b_{i})^{\top} y_{i}\\
                & \leq \max_{i} (b - b_{i})^{\top} y_{i}\\
                & \leq \max_{i} \|y_{i} \| \diam(\mathcal{B})
        \end{align*}
        which concludes the proof.
    \end{proof}

\noindent
Observe that Theorem \ref{thm:ICNN:generalization:perfect} applies to
any subset of data points, which offers tighter guarantees over
smaller domains.

%% file: tex/results.tex
\section{Numerical Experiments}
\label{sec:results}

\newcommand{\ieee}{\texttt{ieee300}}
\newcommand{\pegaseS}{\texttt{pegase1k}}
\newcommand{\pegaseM}{\texttt{pegase2k}}
\newcommand{\rte}{\texttt{rte6k}}

The section reports numerical results for ICNNs that are trained to
learn the value functions of DC-OPF, SOC-OPF, and AC-OPF.  While
DC-OPF and SOC-OPF are convex with convex value functions, AC-OPF is
not convex and its value function is not guaranteed to be convex. In
this last case, the task is thus to approximate (possibly) the OPF
value function with a convex function.

\subsection{Experimental Setting}
\label{sec:results:setup}

The proposed approach is evaluated on several test cases from PGLib
\cite{babaeinejadsarookolaee2019power} with up to 6468 buses.  These
test cases are 50 times larger in size compared to those in
\cite{Zhang2022_ConvexNNSolverDCOPF}. Moreover, these prior results
only considered DC-OPF. 

Table \ref{table: cases} reports, for each system, the number of buses
$|\mathcal{N}|$, branches $|\mathcal{E}|$ and generators
$|\mathcal{G}|$, as well as the nominal total demand
(${\text{P}}^{d}_{\text{ref}}$) and its range across the dataset
($[\underline{\text{P}}^{d}, \bar{\text{P}}^{d}]$).

    \begin{table}[!t]
        \centering
        \caption{Statistics of the PGLib test cases.}
        \label{table: cases}
        \begin{tabular}{lrrrrr}
            \toprule
            \multicolumn{1}{c}{System}
            & \multicolumn{1}{c}{\textbf{$|\mathcal{N}|$}}  
            & \multicolumn{1}{c}{\textbf{$|\mathcal{E}|$}}  
            & \multicolumn{1}{c}{\textbf{$|\mathcal{G}|$}}
            & \multicolumn{1}{c}{${\text{P}}^{d}_{\text{ref}}$}
            & \multicolumn{1}{c}{$[\underline{\text{P}}^{d}, \bar{\text{P}}^{d}]$}
            \\
            \midrule
            \ieee
                & 300  
                & 411  
                & 69   
                & 263
                & [\phantom{0}210, \phantom{0}280]
            \\
            \pegaseS
                & 1354  
                & 1991  
                & 260     
                & 781
                & [\phantom{0}625, \phantom{0}820] 
            \\
            \pegaseM 
                & 2869  
                & 4582  
                & 510      
                & 1522
                & [1218, 1599]
            \\
            \rte 
                & 6468  
                & 9000  
                & 399      
                & 1109
                & [\phantom{0}887, 1164]
            \\
            \bottomrule
        \end{tabular}
    \end{table}
    
For each system, a dataset of OPF instances is generated by perturbing the load vectors as follows
\begin{align*}
\pd = \alpha \times \eta \times \pd_{\text{ref}}, \quad \qd = \alpha \times \eta \times \qd_{\text{ref}},
\end{align*}
where $\alpha$ is the system-wide scaling factor, sampled from a
uniform distribution, $\eta \in \mathbb{R}^{|\mathcal{N}|}$ is the
bus-level uncorrelated noise vector, sampled from a log-normal
distribution with mean $1$ and standard deviation $5\%$, and
$\pd_{\text{ref}}, \qd_{\text{ref}}$ are the nominal active and
reactive load vectors.  The range of $\alpha$ is selected to avoid
regions where AC-OPF becomes infeasible.

For each system, $50,000$ OPF instances are generated and solved.  The
OPF problems are formulated with \textit{PowerModels.jl} and solved with Mosek (DC
and SOC) or Ipopt (AC).  Infeasible instances are excluded, and the
remaining dataset is partitioned into training (40\%), validation
(30\%), and testing (30\%) sets for appropriate training and model
assessment procedures.  For each system and OPF formulation, the paper
trains a DNN and an ICNN to predict the value function of the
considered OPF. Both models use the very same architecture, the only
difference being that DNN weights are unrestricted.  All ML models are
implemented in Julia using \textit{Flux.jl} \cite{Flux.jl-2018}.
Experiments are carried out on Intel(R) Xeon(R) Gold 6226 CPU @
2.70GHz machines with NVIDIA Tesla A100 GPUs on the Phoenix cluster
\cite{PACE}.

\subsection{ICNN Performance}
\label{sec:results:performance}

The paper uses relative absolute optimality gap to measure the
performance of the DNN and ICNN models. Let $z^{*}$ and $\tilde{z}$
denote the ground truth (obtained by the optimization solver) and the
predicted optimal value by a ML model, respectively.  The relative
absolute optimality gap, henceforth referred to as {\em optimality
  gap} for simplicity, is defined as
    \begin{align}
        \label{eq:results:relative_absolute_gap}
        \text{gap} &= \frac{|\tilde{z} - z^{*}|}{|z^{*}|}.
    \end{align}
Note that the prediction $\tilde{z}$ may {\em over-estimate or
  under-estimate} the ground truth $z^{*}$, which is why the absolute
value is needed.  Unless specified otherwise, reported averages use
the geometric mean $\mu(x_{1}, ..., x_{n}) = \sqrt[n]{x_{1} \times
  ... \times x_{n}}$.

\begin{table}[!t]
    \centering
    \caption{ICNN Performance Results.}
    \label{table: results_icnn}
    \begin{tabular}{lrrrrr} 
        \toprule
            &
            & \multicolumn{2}{c}{Mean gap (\%)}
            & \multicolumn{2}{c}{Worst gap (\%)}\\
        \cmidrule(lr){3-4} \cmidrule(lr){5-6}
        System 
            & OPF 
            & \multicolumn{1}{c}{ICNN}      
            & \multicolumn{1}{c}{DNN}       
            & \multicolumn{1}{c}{ICNN}      
            & \multicolumn{1}{c}{DNN}       
        \\
        \midrule
        \ieee
            & DC  
                & 0.15    
                & 0.19   
                & 1.58    
                & 1.90   
            \\
            & SOC 
                & 0.31    
                & 0.37   
                & 5.77    
                & 4.96   
            \\
            & AC  
                & 0.39    
                & 0.39   
                & 15.81    
                & 14.56   
            \\
        \midrule
        \pegaseS
            & DC  
                & 0.28    
                & 0.33  
                & 2.35    
                & 1.83  
            \\
            & SOC 
                & 0.33  
                & 0.82   
                & 2.22  
                & 2.15   
            \\
            & AC  
                & 0.33    
                & 0.68   
                & 2.37    
                & 1.95   
            \\
        \midrule
        \pegaseM
            & DC  
                & 0.22 
                & 0.30 
                & 3.45  
                & 3.21  
            \\
            & SOC 
                & 1.03
                & 0.32 
                & 2.15 
                & 2.36  
            \\
            & AC  
                & 0.24 
                & 0.27 
                & 3.02 
                & 8.89  
            \\
        \midrule
        \rte
            & DC  
                & 0.27 
                & 0.38 
                & 1.76  
                & 1.15  
            \\
            & SOC 
                & 0.29 
                & 0.57 
                & 1.69 
                & 5.52  
            \\
            & AC  
                & 0.25 
                & 0.33 
                & 2.71  
                & 3.08  
            \\
        \bottomrule
    \end{tabular}
    \end{table}

Table \ref{table: results_icnn} reports, for each system and model
architecture (DNN or ICNN), the geometric mean optimality gap, and the
worst-case optimality gap across the testing set. First, observe that,
except for SOC on the \pegaseM{} system (for which ICNN training did
not converge), {\em ICNN always yields a mean optimality gap below
  0.5\%}, and DNN yields mean gaps below 1\%.  Overall, predictions
for DC-OPF tend to be slightly more accurate than those for SOC- and
AC-OPF, which may be explained by the fact that DC-OPF is a linear
programming problem, whereas SOC- and AC-OPF are nonlinear.  Second,
{\em ICNN almost always outperforms DNN in terms of mean optimality
  gap.} This holds even for AC-OPF, whose value function is
non-convex.  {\em These results suggest that ICNN models are accurate
  enough to be used instead of general, non-convex DNN models to
  accurately represent the value functions of OPF problems.}

It is also important to study the worst-case optimality gaps,
especially if ML models are to be embedded in larger optimization
problems.  Interestingly, the worst-case gaps are highest for the
\ieee{} system, with both ICNN and DNN exhibiting worst-case
optimality gaps of about 15\%, i.e., roughly 40 times larger than
their mean gap.  This demonstrates that low mean optimality gaps are
not a guarantee of uniformly good performance.  For the larger
systems, both architectures exhibit lower worst-case gaps, ranging
around 2--5\%, or roughly 20 times larger than the mean optimality
gap.  Interestingly, ICNN tends to produce better worst-case optimality
gaps than DNN for SOC and AC-OPF on the two larger test cases.
Designing training procedures that reduce worst-case performance is an
active area of research.
  
Figure \ref{fig:results:ICNN_vs_DNN_errors} provides a more granular
picture of the DNN and ICNN performance on the \rte{} system.  The
figure reports, for each OPF formulation, the distribution of relative
optimality gaps over the dataset (left panel) and as a function of
total load (right panel).  Note that these results are {\em relative
  optimality gaps} $(\tilde{z} - z^{*}) / |z^{*}|$: a positive
(resp. negative) gap indicates that the ML model over-estimates
(resp. under-estimates) the ground truth value.  Overall, echoing the
results of Table \ref{table: results_icnn}, {\em the ICNN gaps
  exhibit lower variance than the DNN gaps, especially for SOC-OPF.}
Furthermore, both formulations exhibit similar behavior with respect
to the system total load, with overall higher gaps towards the ends of
the range. For SOC-OPF and AC-OPF, the plots on the right panels
highlight the benefits of ICNNs. 
    
    \begin{figure}[!t]
        \centering
        \subfloat[Comparison of errors on DC-OPF]{
            \includegraphics[width=0.48\columnwidth]{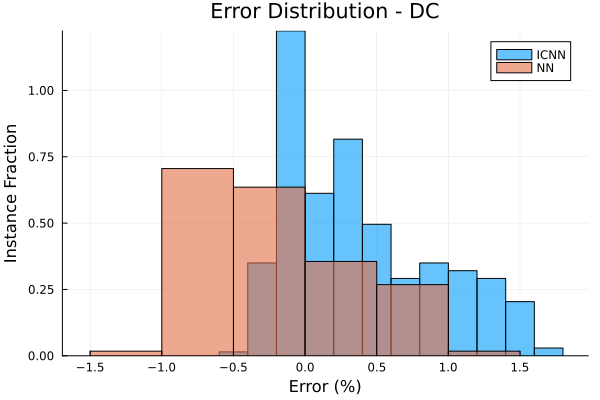}
            \hfill
            \includegraphics[width=0.48\columnwidth]{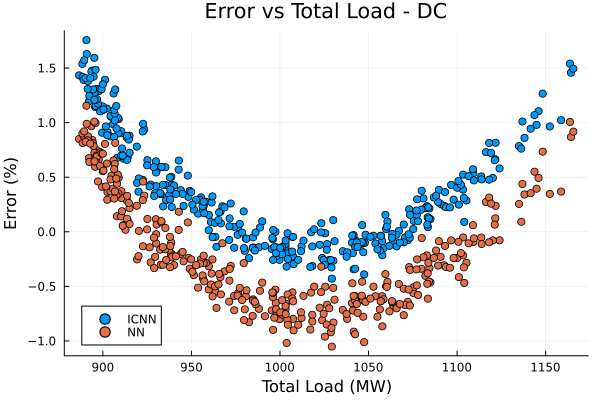}
        }\\
        \subfloat[Comparison of errors on SOC-OPF]{
            \includegraphics[width=0.48\columnwidth]{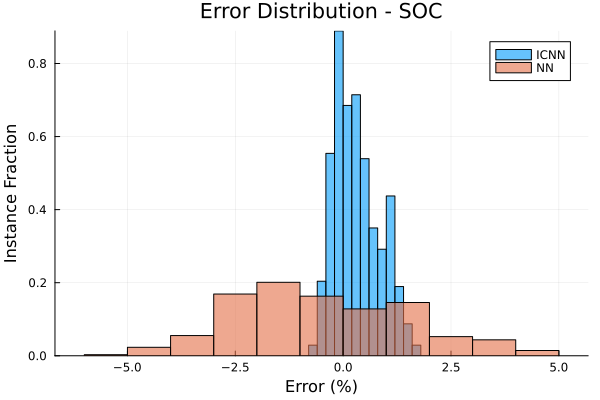}
            \hfill
            \includegraphics[width=0.48\columnwidth]{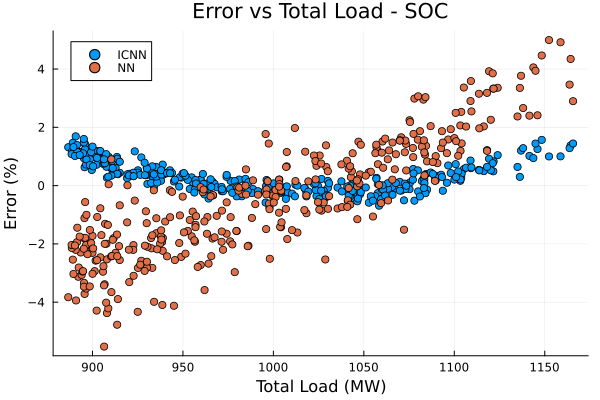}
        }\\
        \subfloat[Comparison of errors on AC-OPF]{
            \includegraphics[width=0.48\columnwidth]{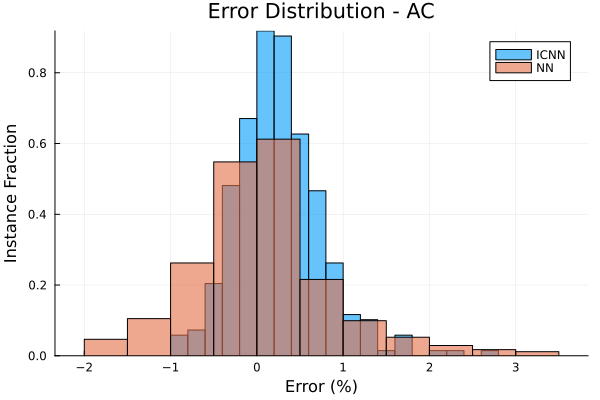}
            \hfill
            \includegraphics[width=0.48\columnwidth]{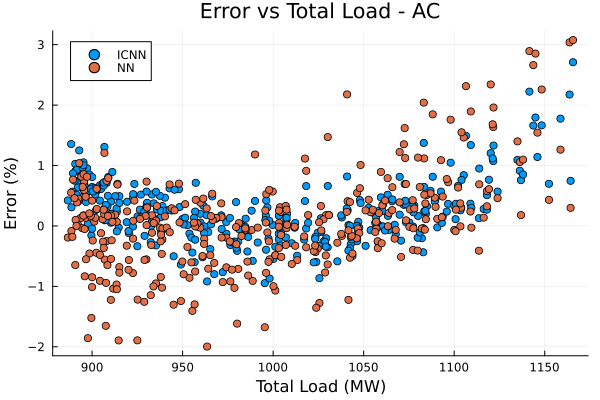}
        }
        \caption{Comparison of relative gap $(\tilde{z} - z^{*})/|z^{*}|$ on \rte. Positive (resp. negative) values mean the prediction over- (resp. under-) estimates the ground truth optimal value.}
        \label{fig:results:ICNN_vs_DNN_errors}
    \end{figure}